\PassOptionsToPackage{preprint}{aistats2027}
\def\arxivversion{}
\documentclass[twoside,letterpaper]{article}
\PassOptionsToPackage{table}{xcolor}
\usepackage{aistats2027}
\usepackage[round]{natbib}
\setcitestyle{citesep={,}}
\renewcommand{\bibsection}{\subsubsection*{References}}
\usepackage[T1]{fontenc}
\usepackage{amsmath,amssymb,amsthm}
\usepackage{graphicx,booktabs,tabularx,longtable,array}
\usepackage{float,placeins}
\usepackage{algorithm,algpseudocode}
\usepackage[table]{xcolor}
\usepackage{tikz}
\usepackage{marginnote}
\usetikzlibrary{arrows.meta,positioning}
\usepackage{pgfplots}
\usepgfplotslibrary{groupplots}
\pgfplotsset{compat=1.18}
\tikzset{every picture/.append style={font=\normalfont\scriptsize}}
\pgfplotsset{every axis/.append style={
  tick label style={font=\normalfont\scriptsize},
  label style={font=\normalfont\scriptsize},
  title style={font=\normalfont\scriptsize}}}
\usepackage{microtype}
\usepackage{hyperref,url}
\hypersetup{hidelinks}
\graphicspath{{./}{figures/}}
\newcommand{\method}{\textnormal{\textsc{NeuralZip}}}
\newcommand{\zipnn}{\textnormal{\textsc{ZipNN}}}
\newcommand{\dfloat}{\textnormal{\textsc{DFloat11}}}
\newcommand{\huffllm}{\textnormal{\textsc{Huff-LLM}}}
\newcommand{\fast}{\textnormal{\textsc{Sharing Only}}}
\newcommand{\ratio}{\textnormal{\textsc{Packing Only}}}

\newtheorem{lemma}{Lemma}
\newtheorem{theorem}{Theorem}
\newcounter{marginnotecount}

\begin{document}
\raggedbottom
\ifdefined\arxivversion
\runningtitle{NeuralZip: Reusable Codebooks for Model Compression}
\runningauthor{Bravo, Horváth, Navarro, and Abeliuk}
\fi
\twocolumn[
\aistatstitle{NeuralZip: Reusable Setup for Fast Lossless Compression}
\ifdefined\arxivversion
\aistatsauthor{Martín Bravo\textsuperscript{1,2,3,*} \And Samuel Horváth\textsuperscript{1} \And Gonzalo Navarro\textsuperscript{2} \And Andrés Abeliuk\textsuperscript{2,3}}
\aistatsaddress{\textsuperscript{1}Mohamed bin Zayed University of Artificial Intelligence, Abu Dhabi, UAE\\
\textsuperscript{2}Department of Computer Science, University of Chile, Santiago, Chile\\
\textsuperscript{3}National Center for Artificial Intelligence (CENIA), Santiago, Chile}
\else
\aistatsauthor{Anonymous authors}
\aistatsaddress{Anonymous institution}
\fi
]
\ifdefined\arxivversion
{\renewcommand{\thefootnote}{*}\footnotetext{Correspondence to: \href{mailto:martin.bravo@mbzuai.ac.ae}{\texttt{martin.bravo@mbzuai.ac.ae}}.}}
\fi
\begin{abstract}

Lossless compression can reduce the storage and movement of model weights without changing their floating-point values, but repeated statistical analysis and code construction add computational overhead. We study whether the statistical structure of exponents can be prepared once and reused. For this, we introduce \method{}, which groups chunks with similar exponent distributions, shares Huffman codes, and selectively represents recurring exponent tuples using packed exponents, thereby achieving additional moderate compression ratios. A setup chooses these representations before subsequent encodings, while every encoding still processes the current tensor values. In floating-point model checkpoints, post-setup compression is 1.81--21.33$\times$ faster than the baselines and achieves exact bit-to-bit reconstruction.
We show that this setup can be precomputed and transferred from another compatible architecture, preserving similar compression ratios and avoiding the need to amortize setup costs.
Therefore, compression adaptation is transferable and reusable. Training checkpoints demonstrate continued reuse as the weights evolve. Finally, GPU experiments reduce active memory usage by up to 27.5\% while reproducing the logits exactly.


\end{abstract}
\section{Introduction}
\label{sec:paper-introduction}
Modern society increasingly relies on intelligent systems, which have motivated the training of models at a very large scale. These models solve problems in areas such as medicine~\citep{singhal2023large}, education~\citep{kasneci2023chatgpt}, and software development~\citep{chen2021evaluating}. Foundation models, in particular, have undergone a rapid transformation in recent years, with substantial improvements in performance, precision, and practical usefulness~\citep{bommasani2021opportunities,achiam2023gpt}.

The exponential growth in the size of foundation models has introduced significant challenges in storage and energy consumption~\citep{hershcovitch2025zipnn,patterson2021carbon}. 
Given this, compression has become a natural partial solution to this problem.
For example, deployment is difficult on devices with limited resources, such as mobile phones, embedded systems, local agents, and edge environments.
Compression can allow resource-constrained edge devices to load models that would otherwise be too large.
Similarly, model hubs distribute checkpoints to inference machines at increasing scale. Nearly three million public model repositories and approximately two billion downloads across the Qwen family were recorded during the first seven months of 2026~\citep{hf_state_of_open_models_summer2026}.
Therefore, it is important to find techniques that allow compression to be fast and efficient.

To address this, several compression techniques have been developed, such as \textit{quantization}~\citep{gholami2021survey,gray1984vector},
which reduces numerical precision to decrease the model size and accelerate inference.
\textit{Pruning}~\citep{lecun1989optimal,zhu2017prune}
removes low-relevance connections or parameters. 
\textit{Knowledge distillation}~\citep{gou2021knowledge}
trains a smaller model to imitate a larger one.
\textit{Low-rank factorization}~\citep{saha2024caldera,sundrani2024llrc}
and parameter sharing reduce redundancy in model parameters,
and \textit{sparse or selectively activated models}~\citep{ma2021sparsification}
use only a fraction of their parameters during inference.

These techniques introduce important limitations. First, they do not preserve the original model: after compression, the exact original weights cannot be recovered unless a copy is stored. Second, the reduction is lossy, so it can degrade accuracy or generation quality. Third, they may require additional training, which adds computational cost. These limitations motivate studying compression methods that reduce storage while keeping the original parameters exactly recoverable and enabling fast compression of models already available in model hubs and on end devices.

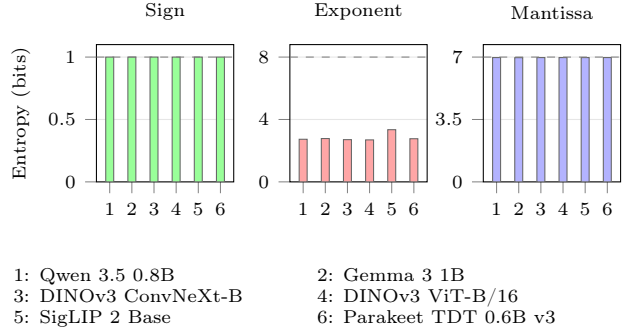
\begin{figure}[t]
\centering
\begin{tikzpicture}
\begin{groupplot}[
    group style={group size=3 by 1, horizontal sep=0.09\linewidth},
    width=0.22\linewidth, height=0.22\linewidth, scale only axis,
    tick label style={font=\normalfont\scriptsize}, label style={font=\normalfont\scriptsize},
    title style={font=\normalfont\scriptsize},
    xmin=-0.6, xmax=5.6, ymin=0,
    xtick={0,1,2,3,4,5},
    xticklabels={1,2,3,4,5,6},
    tick align=outside, tick pos=left,
    ymajorgrids, grid style={gray!20},
]
\nextgroupplot[title={Sign}, ylabel={Entropy (bits)},
    ymax=1.1, ytick={0,0.5,1}]
\addplot[ybar, bar width=3pt, draw=black!60, fill=green!40]
    table[col sep=comma, x expr=\coordindex, y=sign]{data/plots/zero_order_entropy.csv};
\addplot[gray, dashed, sharp plot, domain=-0.6:5.6, samples=2] {1};
\nextgroupplot[title={Exponent}, ymax=8.8, ytick={0,4,8}]
\addplot[ybar, bar width=3pt, draw=black!60, fill=red!35]
    table[col sep=comma, x expr=\coordindex, y=exponent]{data/plots/zero_order_entropy.csv};
\addplot[gray, dashed, sharp plot, domain=-0.6:5.6, samples=2] {8};
\nextgroupplot[title={Mantissa}, ymax=7.7, ytick={0,3.5,7}]
\addplot[ybar, bar width=3pt, draw=black!60, fill=blue!30]
    table[col sep=comma, x expr=\coordindex, y=mantissa]{data/plots/zero_order_entropy.csv};
\addplot[gray, dashed, sharp plot, domain=-0.6:5.6, samples=2] {7};
\end{groupplot}
\end{tikzpicture}
\par\smallskip
{\normalfont\scriptsize\setlength{\tabcolsep}{2pt}
\begin{tabular}{@{}p{0.47\linewidth}p{0.47\linewidth}@{}}
1: Qwen 3.5 0.8B & 2: Gemma 3 1B \\
3: DINOv3 ConvNeXt-B & 4: DINOv3 ViT-B/16 \\
5: SigLIP 2 Base & 6: Parakeet TDT 0.6B v3
\end{tabular}\par}
\caption{Empirical entropy by floating-point field in BF16.}
\label{fig:skew}
\label{fig:ablation-zero-order-entropy}
\end{figure}

A notable example of a \emph{lossless} compressor is \zipnn{}~\citep{hershcovitch2025zipnn}, which exploits the skewed distribution of floating-point exponents and compresses them using Huffman coding. \zipnn{} reports strong compression results across vision and language models while preserving the original parameters exactly. Another relevant approach is \dfloat{}~\citep{zhang2025dfloat11}, which introduces a dynamic-length floating-point representation for efficient GPU inference while preserving the model parameters exactly.
These methods show that lossless compression is both feasible and useful, but their compression procedures can incur substantial computational overhead. For example, recomputing similar codebooks as if they were different distributions leaves an open question: Can we exploit all the statistical regularity available in floating-point weights?

To address this, we identify and empirically characterize two useful properties. First, exponent distributions vary among parameter blocks while forming similar groups, a property we study here as \textit{the heterogeneity of the exponents}. Sharing precomputed codebooks across these groups can accelerate compression, but may increase the exponent storage ratio. To offset this cost, we show that grouping adjacent exponent values into new symbols can reveal local regularities, a property we call \textit{local exponent redundancy}.

Using this, we propose \method{} to take advantage of the two phenomena observed in the evaluated architectures. We combine both ideas in one algorithm: Jensen--Shannon clustering enables Huffman-code sharing while activating packed exponents. On top of that, we show that this setup is sharable across different model architectures. In summary, our contributions\footnote{\ifdefined\arxivversion We provide the code in the \href{https://github.com/MartinEBravo/nzip}{\mbox{GitHub repository}}.\else We provide the code in the anonymized repository: \url{https://anonymous.4open.science/r/neuralzip/}\fi} are:

\begingroup
\setlength{\topsep}{0pt}
\setlength{\partopsep}{0pt}
\setlength{\parskip}{0pt}
\begin{enumerate}
\setlength{\itemsep}{0pt}
\setlength{\parsep}{0pt}
\setlength{\parskip}{0pt}
\item We show evidence of different exponent distribution groups and local redundancy in floating-point model weights.
\item We propose a lossless compressor that combines shared codebooks with exponent packing to reduce repeated computation while retaining moderate compression ratios.
\item We propose a setup-transfer mechanism that enables compression of compatible target models without target-side codebook construction.
\end{enumerate}
\endgroup

\section{Related work}
\label{sec:paper-related}

\paragraph{Model compression.}
Pruning~\citep{zhu2017prune}, quantization~\citep{gholami2021survey}, and
distillation~\citep{gou2021knowledge} reduce model cost by changing parameters
or architecture. Lossless compression instead preserves the chosen weight
representation bit for bit. Whereas general-purpose compressors such as
\textsc{zlib} and \textsc{zstd}~\citep{deutsch1996zlib,collet2018zstandard}
combine dictionary matching with entropy coding, weight-aware methods expose
floating-point structure explicitly. \zipnn{}~\citep{hershcovitch2025zipnn}
reorders bits and bytes before local Huffman coding. LMC~\citep{waddington2025lmc}
combines byte grouping with block-adaptive parallel Huffman coding for
snapshots and checkpoint deltas. Both adapt codes to input blocks, motivating
the question of when code construction can be shared rather than repeated.

\paragraph{Codebook sharing.} Shannonic~\citep{ibrahim2026shannonic} also prepares reusable coding tables,
using range-partitioned tANS for communication of eight-bit quantized tensors 
in federated learning.
Its evaluated representation differs from our bit-exact floating-point weights.
Thus, precomputed tables are an established idea. The distinction here is how
sharing groups are discovered.
Sharing marginal distributions does not itself exploit dependencies between
adjacent exponents. MM-RePair~\citep{ferragina2022improving} exploits repeated
patterns through grammars and supports matrix--vector multiplication over
the compressed representation. Rather than constructing a recursive grammar,
\method{} adds frequent packed exponent symbols to a shared Huffman alphabet alongside
singletons. This combines distribution sharing with local redundancy while
retaining exact reconstruction.

\paragraph{Compressed Inference.}
Compressed execution poses a complementary problem: reconstructing weights
quickly enough to make memory savings useful. NeuZip~\citep{hao2024neuzip}
uses ANS-coded BF16 exponents for lossless training, with a separate lossy
mantissa variant. \dfloat{}~\citep{zhang2025dfloat11} makes BF16 Huffman
reconstruction practical on GPUs through lookup and restart structures.
ECF8~\citep{yang2026ecf8} couples exponent coding with tensor management for
weights already in FP8. Huff-LLM~\citep{yubeaton2025huff} integrates Huffman
decompressors into custom accelerators. These works establish exponent
compression and on-demand reconstruction, on which our GPU inference builds.

\begin{table}[H]
\centering
\scriptsize
\setlength{\tabcolsep}{3pt}
\begin{tabularx}{\linewidth}{*{3}{>{\raggedright\arraybackslash}X}}
\toprule
Compressor & Coding mechanism & Primary objective \\
\midrule
\zipnn{} \citep{hershcovitch2025zipnn} & Local Huffman codes with raw fallback & Model storage and transfer \\
LMC \citep{waddington2025lmc} & Parallel Huffman coding & High-throughput checkpointing \\
\midrule
\textsc{Shannonic} \citep{ibrahim2026shannonic} & Range-partitioned tANS for eight-bit tensors & Federated and edge--cloud communication \\
\midrule
\textsc{NeuZip} \citep{hao2024neuzip} & ANS-coded BF16 exponents & Memory-efficient training and inference \\
\dfloat{} \citep{zhang2025dfloat11} & BF16 exponent Huffman coding with GPU lookup structures & GPU inference \\
\textsc{ECF8} \citep{yang2026ecf8} & Entropy-coded exponents of FP8 weights & Lossless FP8 inference \\
\huffllm{} \citep{yubeaton2025huff} & Huffman decompressors placed near compute units & Accelerator inference \\
\midrule
\textsc{MM-RePair} \citep{ferragina2022improving} & RePair grammar over a matrix representation & Matrix--vector multiplication over compressed matrices \\
\bottomrule
\end{tabularx}
\caption{Lossless Model Compression comparison.}
\label{tab:related-lossless-compressors}
\end{table}

\begin{figure*}[t]
\centering
\begin{minipage}[t]{0.48\textwidth}
\vspace{0pt}
\centering
\makebox[0.33\linewidth]{\includegraphics{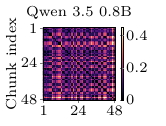}}%
\makebox[0.33\linewidth]{\includegraphics{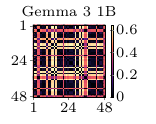}}%
\makebox[0.33\linewidth]{\includegraphics{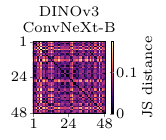}}
\par
\makebox[0.33\linewidth]{\includegraphics{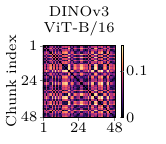}}%
\makebox[0.33\linewidth]{\includegraphics{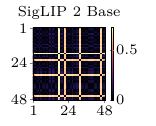}}%
\makebox[0.33\linewidth]{\includegraphics{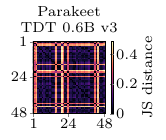}}
\par\vspace{0.1cm}
{\normalfont\scriptsize Chunk index\par}
\caption{Jensen--Shannon distances between exponent histograms of
48 sampled chunks per model.}
\label{fig:paper-js}
\end{minipage}\hfill
\begin{minipage}[t]{0.48\textwidth}
\vspace{0pt}
\centering
\begin{tikzpicture}
\definecolor{qwenline}{HTML}{1F77B4}
\definecolor{gemmaline}{HTML}{FF7F0E}
\definecolor{dinoconvnextline}{HTML}{2CA02C}
\definecolor{dinovitline}{HTML}{D62728}
\definecolor{siglipline}{HTML}{9467BD}
\definecolor{parakeetline}{HTML}{8C564B}
\begin{axis}[
    width=0.85\linewidth, height=0.35\linewidth, scale only axis,
    tick label style={font=\normalfont\scriptsize}, label style={font=\normalfont\scriptsize},
    title style={font=\normalfont\scriptsize},
    xlabel={Packing size}, ylabel={Entropy (bits)},
    xmin=0.65, xmax=8.35, ymin=2.28, ymax=3.4,
    xtick={1,2,4,8}, ytick={2.4,2.6,2.8,3.0,3.2,3.4},
    tick align=outside, tick pos=left, ymajorgrids, grid style={gray!25},
    legend cell align=left, legend columns=3,
    legend style={at={(0.5,-0.32)}, anchor=north,
        draw=gray!35, fill=white, font=\normalfont\scriptsize},
    every axis plot/.append style={thick, dashed, mark=*, mark size=1.6pt},
    unbounded coords=discard, filter discard warning=false,
]
\addplot[qwenline, restrict expr to domain={\coordindex}{0:3}]
    table[col sep=comma, x=n, y=bits_per_exponent]{data/plots/ngram_exponent_entropy.csv};
\addlegendentry{\shortstack[l]{Qwen 3.5\\0.8B}}
\addplot[gemmaline, restrict expr to domain={\coordindex}{4:7}]
    table[col sep=comma, x=n, y=bits_per_exponent]{data/plots/ngram_exponent_entropy.csv};
\addlegendentry{Gemma 3 1B}
\addplot[dinoconvnextline, restrict expr to domain={\coordindex}{8:11}]
    table[col sep=comma, x=n, y=bits_per_exponent]{data/plots/ngram_exponent_entropy.csv};
\addlegendentry{\shortstack[l]{DINOv3\\ConvNeXt-B}}
\addplot[dinovitline, restrict expr to domain={\coordindex}{12:15}]
    table[col sep=comma, x=n, y=bits_per_exponent]{data/plots/ngram_exponent_entropy.csv};
\addlegendentry{\shortstack[l]{DINOv3\\ViT-B/16}}
\addplot[siglipline, restrict expr to domain={\coordindex}{16:19}]
    table[col sep=comma, x=n, y=bits_per_exponent]{data/plots/ngram_exponent_entropy.csv};
\addlegendentry{SigLIP 2 Base}
\addplot[parakeetline, restrict expr to domain={\coordindex}{20:23}]
    table[col sep=comma, x=n, y=bits_per_exponent]{data/plots/ngram_exponent_entropy.csv};
\addlegendentry{\shortstack[l]{Parakeet TDT\\0.6B v3}}
\end{axis}
\end{tikzpicture}
\par\vspace{-18pt}
\caption{Local exponent redundancy entropy.}
\label{fig:ablation-ngram-entropy-gain}
\end{minipage}
\end{figure*}

\section{Preliminaries}
\label{sec:paper-statistics}

We now show the property that allows lossless compressors such as \zipnn{} and \dfloat{} to work, and we introduce the idea of exponent skewing. For a distribution $p$ on a finite alphabet $\mathcal{A}$, let us define the entropy
$H(p)=-\sum_{a \in \mathcal{A}} p(a)\log p(a)$, where logarithms are base 2. Entropy measures the expected amount of information needed to describe the state of a variable~\citep{navarro2016compact}.

\textbf{Skewed Exponent field.} We measure the entropy of the sign, exponent, and mantissa fields of the weights of BF16
models shown in Figure~\ref{fig:ablation-zero-order-entropy}. The exponent entropy is $2.691$-$3.345$ bits, while the sign and
mantissa entropies remain near $1$ and $6.97$ bits. For BF16, we use 1 bit for the sign field, 8 bits for the exponent, and 7 bits for the mantissa. Therefore, the sign and the mantissa fully utilize the bits, but the exponent, on average, uses only 2.7-3.4 out of 8 bits. Thus, entropy-based compressors, such as Huffman codes, can take advantage of this property.

Exponent concentration is an empirical observation in the models studied here, our hypothesis is that it is a property caused by the optimizer. In Appendix~\ref{tab:optimizers}, we measure how exponent entropy changes during early training under several optimizer configurations.

One could take a linear projection of all the exponents of the model weights and compress them using a single Huffman code, but that might not be the most efficient solution.
For any non-empty finite-alphabet sequence
$S$, define $\widehat p_S(a)=|\{i:S_i=a\}|/|S|$ and the normalized entropy
$H_0(S)=H(\widehat p_S)$. The exponent alphabet is
$\Sigma_E=\{0,\ldots,2^E-1\}$, where $E\geq1$ is an integer.
For a partition $\mathcal P$ into non-empty ordered subsequences, the ideal zero-order coding cost in bits is:
\[
C_0(\mathcal P)=\sum_{S\in\mathcal P}|S|H_0(S).
\]

\begin{lemma}[Cost monotonicity under refinement]
\label{lem:stream-partition-bound}
If $\mathcal P_\mathcal{C}$ refines $\mathcal P$ by splitting its sequence into chunks, preserving their order, and omitting empty chunks,
then $C_0(\mathcal P_\mathcal{C})\leq C_0(\mathcal P)$.
\end{lemma}
\begin{proof}
For each sequence $S$ split into $S_1,\ldots,S_k$, let
$\lambda_i=|S_i|/|S|$. Symbol counts give
$\widehat p_S=\sum_i\lambda_i\widehat p_{S_i}$, so the concavity of entropy~\citep{cover2006elements} yields
$|S|H_0(S)\geq\sum_i|S_i|H_0(S_i)$. Summing over $S\in\mathcal P$
proves the claim.
\end{proof}

Methods such as \zipnn{} exploit this property by dividing tensors into small chunks of exponents and compressing each chunk using Huffman coding. This allows the codebooks to reflect local exponent distributions rather than relying on a single global codebook.
In Appendix~\ref{tab:level-selection}, we explore how much we save if we vary the refinement of these partitions.
%
The problem with this approach is that it must create a codebook for each chunk, ignoring statistical structures that may arise in tensors and could save the recomputation of equivalent codebooks. Next, we introduce these additional properties.

\textbf{Heterogeneity of the exponents.}
For the fraction $u$ with exponent values of $N_u$, let
$p_u(a)=\frac{1}{N_u}\sum_i1(e_i=a)$ be its normalized histogram, where $1(x)$ is the indicator function, which equals 1 if the condition $x$ is true and 0 otherwise. We measure
histogram dissimilarity using the square root of Jensen--Shannon divergence~\citep{lin1991divergence}:
\begin{equation}
 \begin{gathered}
 D_{\rm JS}(p,q)=
 \sqrt{\tfrac12D_{\rm KL}(p\Vert m)+\tfrac12D_{\rm KL}(q\Vert m)},\\
 \text{where }m=\tfrac12(p+q),
 \end{gathered}
 \label{eq:paper-js}
\end{equation}
From the 12 largest tensors in each model, we randomly choose 48 chunks of size 256 KiB among six different architectures (see
Figure~\ref{fig:paper-js}). $14.3$--$49.5\%$ of the comparisons
fall below a distance threshold of $0.05$.
We observe both near-identical and widely separated
histograms.
Lemma~\ref{lem:stream-partition-bound} shows that splitting tensors into chunks does not increase entropy.

To measure actual gains, we employ greedy clustering with a threshold $\tau$ on the Jensen--Shannon distance between a chunk's exponent histogram and compatible cluster centroids, assigning the chunk to the nearest centroid if $D_{JS} \leq \tau$. Sharing codebooks can increase the exponent storage ratio but enables faster compression with precomputed tables (see Table~\ref{tab:paper-transfer}). Therefore, we examine another property that will help reduce storage while retaining precomputed shared codebooks.

\textbf{Local exponent redundancy.}
For the longest prefix whose length is divisible by $n$, we can represent non-overlapping
tuples of $n$ consecutive exponent symbols as compound
symbols. 
By packing two and four values, we lower the empirical entropy by
$0.007$ -- $0.235$ and $0.016$ -- $0.429$ bits per exponent, respectively
(Figure~\ref{fig:ablation-ngram-entropy-gain}). 


Fix an integer $n\geq1$ and assume that every sequence length is divisible by $n$.
For $S=(x_1,\ldots,x_t)$, let $G_n(S)$ contain the $t/n$ non-overlapping tuples
$(x_{(j-1)n+1},\ldots,x_{jn})$, $1\leq j\leq t/n$.
A bijective integer representation of these tuples leaves their entropy
unchanged. Define $H_0^{(n)}(S)=H_0(G_n(S))$ and write the packed coding cost as:
\[
C_0^{(n)}(\mathcal P)
=\sum_{S\in\mathcal P}\frac{|S|}{n}H_0^{(n)}(S).
\]

\begin{lemma}[Cost monotonicity under packing]
\label{lem:exponent-ngram-bound}
For every integer $n\geq1$ that divides the length of each $S\in\mathcal P$,
$C_0^{(n)}(\mathcal P)\leq C_0(\mathcal P)$.
\end{lemma}
\begin{proof}
Draw a tuple uniformly from $G_n(S)$, and let $X_b$ be
its $b$-th component with marginal distribution $p_b$. The joint law is
the empirical tuple distribution, while
$\widehat p_S=n^{-1}\sum_{b=1}^n p_b$. By joint-entropy subadditivity and
concavity, the cost of each sequence satisfies
\[
\begin{aligned}
\frac{|S|}{n}H_0^{(n)}(S)
&\leq\frac{|S|}{n}\sum_{b=1}^n H(p_b)\\
&\leq |S|H_0(S).
\end{aligned}
\]
We sum over $S\in\mathcal P$ to obtain the stated bound.
\end{proof}

For the non-overlapping construction above, we keep any remaining symbols as a raw tail. The next theorem combines both lemmas into a single ideal bound for refined tensors and packed exponents.

\begin{theorem}[Ideal bound]
\label{teo:packed-clustered-bound}
Let $\mathcal P_C$ refine $\mathcal P$ by splitting sequences into chunks.
If every refined sequence has a length divisible by
an integer $n\geq1$, then $C_0^{(n)}(\mathcal P_\mathcal{C})\leq C_0(\mathcal P)$.
\end{theorem}

\begin{proof}
By the assumption of divisibility, Lemmas~\ref{lem:exponent-ngram-bound}
and~\ref{lem:stream-partition-bound} yield
$C_0^{(n)}(\mathcal P_\mathcal{C})\leq C_0(\mathcal P_\mathcal{C})\leq C_0(\mathcal P)$.
\end{proof}

The bound excludes raw tails, escapes, integer code lengths, headers and
metadata, local-Huffman and raw fallbacks, and clustering costs. We must therefore measure
the practical advantages of size and time.

\begin{figure*}[t]
\centering
\begin{minipage}[t]{0.48\textwidth}
\vspace{0pt}
\centering
\begin{tikzpicture}[
  font=\normalfont\scriptsize,
  block/.style={draw=black!65,minimum width=11mm,minimum height=6mm,inner sep=0pt},
  tree node/.style={circle,draw,minimum size=2.6mm,inner sep=0pt},
  connection/.style={line width=0.35pt},
  link/.style={-{Latex[length=1.3mm]},line width=0.5pt}]
\node[anchor=east] at (-0.15,-0.3) {Tensor 1};
\node[anchor=east] at (-0.15,-2.5) {Tensor 2};
\foreach \id/\x/\y/\shade/\number in {
  a/0.55/0/red/1,b/1.65/0/blue/2,c/2.75/0/red/3,
  d/0.55/-0.6/green/4,e/1.65/-0.6/red/5,f/2.75/-0.6/blue/6,
  g/0.55/-2.2/green/1,h/1.65/-2.2/red/2,
  i/0.55/-2.8/red/3,j/1.65/-2.8/green/4}{
  \node[block,fill=\shade!25] (\id) at (\x,\y) {Chunk \number};
}
\foreach \shade/\y in {red/0.7,blue/-1.3,green/-3.0}{
  \node[tree node,draw=\shade!70!black,fill=\shade!35] (\shade) at (5.5,\y) {};
  \node[tree node,draw=\shade!70!black,fill=\shade!35] (\shade-left) at (5.0,\y-0.65) {};
  \node[tree node,draw=\shade!70!black,fill=\shade!35] (\shade-right) at (6.0,\y-0.65) {};
  \draw[\shade!70!black] (\shade) -- (\shade-left);
  \draw[\shade!70!black] (\shade) -- (\shade-right);
}
\foreach \shade/\y in {red/0.7,blue/-1.3,green/-3.0}{
  \coordinate (\shade-join) at (4.65,\y);
  \draw[link,\shade!70!black] (\shade-join) -- (\shade.west);
}
\draw[connection,red!70!black] (a.north)
  .. controls (0.55,1.0) and (3.7,0.7) .. (red-join);
\draw[connection,red!70!black] (c.north)
  .. controls (2.75,0.65) and (3.9,0.7) .. (red-join);
\draw[connection,red!70!black] (e.south)
  .. controls (1.65,-1.25) and (4.1,-1.0) .. (red-join);
\draw[connection,red!70!black] (h.east)
  .. controls (3.1,-2.2) and (4.2,-0.2) .. (red-join);
\draw[connection,red!70!black] (i.south)
  .. controls (0.55,-3.65) and (4.4,-3.65) .. (red-join);
\draw[connection,blue!70!black] (b.north)
  .. controls (1.65,1.4) and (4.2,0.5) .. (blue-join);
\draw[connection,blue!70!black] (f.east)
  .. controls (3.9,-0.6) and (4.0,-1.3) .. (blue-join);
\draw[connection,green!70!black] (d.south)
  .. controls (0.55,-1.4) and (4.0,-1.0) .. (green-join);
\draw[connection,green!70!black] (g.north)
  .. controls (0.55,-1.55) and (3.8,-1.8) .. (green-join);
\draw[connection,green!70!black] (j.south)
  .. controls (1.65,-3.4) and (3.9,-3.0) .. (green-join);
\end{tikzpicture}
\caption{Sharing codebook example.}
\label{fig:shared-codebook-mechanism}
\end{minipage}\hfill
\begin{minipage}[t]{0.48\textwidth}
\vspace{0pt}
\centering
\begin{tikzpicture}[
  font=\normalfont\scriptsize,
  exponent/.style={draw=black!65,minimum width=6.5mm,minimum height=5mm,inner sep=1pt},
  codebook/.style={draw=black!65,rounded corners=2pt,align=center,inner sep=5pt}]
\foreach \i/\value/\shade in {
  0/127/red,1/127/red,2/125/green,3/126/blue,4/127/blue,
  5/125/green,6/127/red,7/127/red,8/126/blue,9/127/blue}{
  \node[exponent,fill=\shade!20] at (0.65*\i,0) {\value};
}
\foreach \x in {0.325,4.225}{
  \node[red!70!black] at (\x,-0.48) {$\underbrace{\hspace{11mm}}_{\mathrm{B}}$};
}
\foreach \x in {1.3,3.25}{
  \node[green!50!black] at (\x,-0.48) {$\underbrace{\hspace{4.5mm}}_{\mathrm{C}}$};
}
\foreach \x in {2.275,5.525}{
  \node[blue!70!black] at (\x,-0.48) {$\underbrace{\hspace{11mm}}_{\mathrm{E}}$};
}
\node[codebook,anchor=north] (book) at (2.925,-0.95)
  {\textbf{Codebook}\\[2pt]
   \begin{tabular}{ccc}
   Token & Meaning & Code\\
   \textcolor{red!70!black}{A} & $127$ & \textcolor{red!70!black}{\texttt{110}}\\
   \textcolor{red!70!black}{B} & $(\mathrm{A},\mathrm{A})$ & \textcolor{red!70!black}{\texttt{00}}\\
   \textcolor{green!50!black}{C} & $125$ & \textcolor{green!50!black}{\texttt{01}}\\
   \textcolor{blue!70!black}{D} & $126$ & \textcolor{blue!70!black}{\texttt{111}}\\
   \textcolor{blue!70!black}{E} & $(\mathrm{D},\mathrm{A})$ & \textcolor{blue!70!black}{\texttt{10}}
   \end{tabular}};
\node[anchor=north] at ([yshift=-3mm]book.south) {Compressed stream:
  \textcolor{red!70!black}{\texttt{00}}\,|\,
  \textcolor{green!50!black}{\texttt{01}}\,|\,
  \textcolor{blue!70!black}{\texttt{10}}\,|\,
  \textcolor{green!50!black}{\texttt{01}}\,|\,
  \textcolor{red!70!black}{\texttt{00}}\,|\,
  \textcolor{blue!70!black}{\texttt{10}}};
\end{tikzpicture}
\caption{Packed exponents example.}
\label{fig:packed-exponent-mechanism}
\end{minipage}
\end{figure*}


\section{Our Method: NeuralZip}
\label{sec:paper-method}
\subsection{A reusable setup}

The setup builds and freezes codebooks before compression. It clusters chunks and packs adjacent exponents. For transfer, it also builds a bank of source-only codebooks to handle changes in exponent distributions. Appendix~\ref{tab:setup-cost} reports the setup costs and computational complexity analysis. 


\paragraph{Codebook sharing.}
We divide the tensors into chunks of at most
$b=256\,$KiB of the source weights. We assign each chunk to the nearest compatible centroid using its
normalized exponent histogram when
$D_{\rm JS}\leq\tau=0.05$. Otherwise, we start a cluster (we show this threshold selection in Appendix~\ref{tab:ablation-js-threshold}). We weight centroids
by the number of values. Compatibility requires the same data type and
exponent layout. Therefore, different chunks within a tensor can use different
codebooks, while chunks from different tensors can share one codebook
(See Figure~\ref{fig:shared-codebook-mechanism}).

\paragraph{Packed Exponents.}
For each cluster vocabulary $V$, we select up to 254 frequent packed symbols,
each formed from two adjacent exponents, counting all positions without crossing chunk boundaries. We augment
the singleton alphabet $\Sigma_E$ with packed symbols. We remove a packed symbol
when its code is longer than the two singleton codes, and recompute token counts and the
table until no further packed symbols need removal.
For a chunk $S=(e_1,\ldots,e_t)$, we start parsing at $i=1$.
If $i<t$ and $(e_i,e_{i+1})\in V$, we emit the corresponding packed symbol
and advance by two. Otherwise, we emit the singleton $e_i$ and advance
by one. Packed alignment can change after a singleton. We also compress any final unmatched
exponent individually (See Figure~\ref{fig:packed-exponent-mechanism}).

\paragraph{Codebook preparation.}
We sample up to $M=8\,$ MiB of exponents per cluster, keeping chunk boundaries.
The adaptive scan counts tokens, with pseudocount one for each singleton and
selected packed symbol. Up to four chunks spread across the source cluster
estimate payload and fragment-record bytes, including raw fallback. We scale
this estimate to the cluster and add codebook and vocabulary bytes. We keep the
pruned mixed table only if its estimated archive is smaller than the singleton
alternative. Otherwise, we use singleton coding or raw storage if no usable
table exists. Sample timings do not affect this choice. We then freeze the
tables and vocabulary.

For transfer, we build a codebook bank from up to eight of the largest source
JS clusters. Source mantissas let us estimate the exponent changes caused by
scaling by $2^f$, where $f\in\{0,1/8,\ldots,7/8\}$. We combine these with integer
exponent shifts from $-8$ to $8$. The bank stores distinct packed tables and
their singleton fallbacks. Before building mixed tables, we add 1\% of the
calibrated singleton counts to the token counts. We use only source data and
do not change the source weights.

\subsection{Compression and exact decompression}
During compression, we read current weights, reversibly extract their exponents, and
preserve sign and mantissa bits verbatim. The adaptive scan encodes both packed
and singleton symbols using the chunk's assigned shared Huffman table.

We can compress chunks in parallel using immutable tables. We store unmapped chunks
raw. We also store a mapped chunk raw if its compressed exponent
payload plus fragment record is not below $\rho=0.95$ times the raw exponent
size, without building a local table.
We record raw flags, lengths, codebook IDs, and tensor layouts in the archive,
including each used codebook and packed vocabulary once.

During decompression, we read raw chunks or their prefix-free Huffman streams. We
restore one exponent from each singleton and two from each packed symbol. Stored exponent counts and frame
lengths delimit chunks. We recombine exponents with the preserved bits to restore
the original floating-point bit patterns.

\subsection{GPU Inference}
\label{sec:paper-gpu-method}
For the GPU path, we keep coded exponents,
verbatim sign/mantissa bits, and lookup/restart
metadata on the GPU. Before each transformer block, we launch a
batched PyTorch CUDA decompressor through a hook. We use prefix lookups to recover exponents and code lengths,
and restart offsets to support parallel decompression. We reconstruct the exact weights
in a reused workspace across blocks and pass them to standard matrix
multiplications. We leave unselected parameters uncompressed and pay the decompression cost on each forward
pass.

\subsection{Setup transfer}

We reuse the source's tensor and global codebooks, but not its chunk assignments.
Before each CPU compression call, we test these default tables on 16 short windows
spread across the model's parameters. Each window contains up to 256 exponents,
parsed with the same greedy rule used for compression. If the average estimated
cost exceeds $\gamma=3$ bits per exponent, we choose codebooks from the bank for
each chunk. We sample about 1024 exponents per chunk to rank the tables, then
compare up to four packed candidates with a singleton fallback. We choose the
one with the lowest estimated cost. The measured compression time includes both
this test and codebook selection. Target samples only choose among frozen
codebooks. They do not build new tables or vocabularies. Raw fallback remains
available. Compression can worsen as exponent distributions
change, but reconstruction remains bit-wise exact.

\begin{table*}[!t]
\centering\small
\caption{\textbf{Model compression and setup reuse. }}
\label{tab:paper-transfer}
\footnotesize
\setlength{\tabcolsep}{1.5pt}
\resizebox{\textwidth}{!}{%
\begin{tabular}{@{}llrrrrrr@{}}
\toprule
Model & Compressor &
\shortstack{Setup\\(s) $\downarrow$} &
\shortstack{Compression\\(s) $\downarrow$} & \shortstack{Exponent\\ratio\\(\%) $\downarrow$} & \shortstack{Total\\memory\\(GB) $\downarrow$} & \shortstack{GPU\\memory\\(GiB) $\downarrow$} & \shortstack{Inference\\(ms) $\downarrow$} \\
\midrule
Qwen 3.5 0.8B & BF16 & -- & -- & 100.000 & 1.747 & 1.625 & 110.58 \\
{\scriptsize\citep{qwen2026qwen35}} & \zipnn{} & -- & 1.889 & 33.256 & 1.164 & -- & -- \\
 & \fast{} & \textbf{15.732} & 0.399 & 34.494 & 1.175 &  \textbf{1.295}&  \textbf{115.53}\\
 & \ratio{} & 42.415 & 0.457 & 33.700 & 1.168 &  1.383&  118.74\\
 & \method{} (own) & 17.976 & 0.440 & \textbf{32.985} & \textbf{1.162} & 1.297 & 120.82 \\
 & \method{} (from Qwen 3.5 2B) & -- & \textbf{0.292} & 33.498 & 1.166 & 1.302 & 118.80 \\
 & \dfloat{} & -- & 276.628 & 35.997 & 1.188 & 1.327 & 118.21 \\
\midrule
Gemma 3 1B & BF16 & -- & -- & 100.000 & 2.000 &  1.920&  106.92\\
{\scriptsize\citep{gemmateam2025gemma3}} & \zipnn{} & -- & 1.500 & 32.972 & 1.330 & -- & -- \\
 & \fast{} & \textbf{14.141} & \textbf{0.282} & 34.565 & 1.346 &  \textbf{1.517}&  \textbf{104.22}\\
 & \ratio{} & 46.717 & 0.351 & 33.441 & 1.334 &  1.620&  108.01\\
 & \method{} (own) & 15.067 & 0.330 & \textbf{32.724} & \textbf{1.327} &  1.518&  105.59\\
  & \method{} (from Gemma 3 4B) & -- & 0.335 & 35.504 & 1.355 &  1.537&  104.97\\
 & \dfloat{} & -- & 313.147 & 35.531 & 1.355 &  1.532&  117.61\\
\midrule
DINOv3 ConvNeXt-B & BF16 & -- & -- & 100.000 & 0.175 &  0.172&  23.28\\
{\scriptsize\citep{simeoni2025dinov3}} & \zipnn{} & -- & 0.408 & 33.943 & \textbf{0.117} & -- & -- \\
 & \fast{} & \textbf{2.432} & \textbf{0.029} & 35.505 & 0.119 &  0.138&  \textbf{15.97}\\
 & \ratio{} & 4.215 & 0.035 & 34.294 & 0.118 &  0.151&  17.78\\
 & \method{} (own) & 2.898 & 0.033 & \textbf{33.687} & \textbf{0.117} &  0.139&  16.80\\
  & \method{} (from DINOv3 ConvNeXt-L) & -- & 0.035 & 33.844 & \textbf{0.117} &  0.140&  16.67\\
 & \dfloat{} & -- & 29.020 & 36.909 & 0.120 &  \textbf{0.137}&  16.36\\
\midrule
DINOv3 ViT-B/16 & BF16 & -- & -- & 100.000 & 0.171 & 0.176 & 20.32 \\
{\scriptsize\citep{simeoni2025dinov3}} & \zipnn{} & -- & 0.305 & 33.704 & 0.115 & -- & -- \\
 & \fast{} & \textbf{1.315} & \textbf{0.026} & 35.415 & 0.116 &  0.134&  18.87\\
 & \ratio{} & 4.050 & 0.034 & 34.196 & 0.115 &  0.146&  19.72\\
 & \method{} (own) & 1.768 & 0.031 & \textbf{33.452} & \textbf{0.114} & 0.134 & \textbf{17.80} \\
  & \method{} (from DINOv3 ViT-L/16) & -- & 0.033 & 34.158 & 0.115 & 0.135 & 19.46 \\
 & \dfloat{} & -- & 27.755 & 36.559 & 0.117 & \textbf{0.132} & 18.30 \\
\midrule
SigLIP 2 Base & BF16 & -- & -- & 100.000 & 0.750 &  0.723&  41.48\\
{\scriptsize\citep{tschannen2025siglip2}} & \zipnn{} & -- & 0.779 & 36.870 & 0.514 & -- & -- \\
 & \fast{} & \textbf{5.782} & \textbf{0.091} & 38.032 & 0.518 &  0.619&  \textbf{31.16}\\
 & \ratio{} & 17.697 & 0.122 & 36.139 & 0.511 &  0.646&  32.65\\
 & \method{} (own) & 7.307 & 0.111 & \textbf{35.383} & \textbf{0.508} &  0.621&  32.69\\
  & \method{} (from SigLIP 2 Large) & -- & 0.114 & 35.866 & 0.510 &  0.623&  33.71\\
 & \dfloat{} & -- & 134.630 & 40.135 & 0.526 &  \textbf{0.617}&  32.95\\
\midrule
Parakeet TDT 0.6B v3 & BF16 & -- & -- & 100.000 & 1.254 & 1.177 & 96.83 \\
{\scriptsize\citep{sekoyan2025parakeet}} & \zipnn{} & -- & 1.796 & 33.477 & 0.837 & -- & -- \\
 & \fast{} & \textbf{18.539} & \textbf{0.149} & 34.826 & 0.846 &  \textbf{0.907}&  103.24\\
 & \ratio{} & 29.336 & 0.197 & 33.902 & 0.840 &  0.986&  104.46\\
 & \method{} (own) & 20.224 & 0.183 & \textbf{33.213} & \textbf{0.835} & 0.909 & \textbf{102.31} \\
  & \method{} (from Parakeet TDT 0.6B v2) & -- & 0.188 & 33.873 & 0.840 & 0.916 & 104.41 \\
 & \dfloat{} & -- & 226.298 & 36.215 & 0.854 & 0.938 & 104.92 \\
\bottomrule
\end{tabular}
}
\end{table*}

\section{Experiments}
\label{sec:paper-experiments}
We evaluate four research questions:

\begingroup
\setlength{\topsep}{0pt}
\setlength{\partopsep}{0pt}
\setlength{\parskip}{0pt}
\begin{enumerate}
\setlength{\itemsep}{0pt}
\setlength{\parsep}{0pt}
\setlength{\parskip}{0pt}
    \item \textit{Can we accelerate exact compression with a reusable setup while preserving the compression rate?}
    \item \textit{When can we reduce the representation size through exponent packing?}
    \item \textit{Can we reuse a setup across related models or evolving weights?}
    \item \textit{Can we obtain a favorable memory–latency trade-off with compressed GPU inference?}
\end{enumerate}
\endgroup
We provide complete measurements, including unfavorable outcomes, in the appendices.

\subsection{Experimental setup}
\label{sec:paper-protocol}

We pass models in memory to each compressor, use the same pinned CPU cores, and exclude loading, disk I/O, and bit-wise verification from timing.
The model compression experiment was evaluated on a cluster with two 16-core EPYC 7343 CPUs (32 physical cores,
64 threads, and two NUMA nodes) and 1 TiB of RAM. We use 12 CPU threads for compression unless otherwise specified.
GPU benchmarks were evaluated on four ASUS Ascent GX10 desktop supercomputers configured with a 20-core Arm processor (comprising 10 Cortex-X925 and 10 Cortex-A725 cores) and an integrated NVIDIA Blackwell GPU, using 128 GB of coherent unified LPDDR5x system memory at 273 GB/s bandwidth.
We report medians after warm-up runs: 20 CPU measurements for Table~\ref{tab:paper-transfer} and ablations, 9 for Table~\ref{tab:model-compression}, and 5 GPU trials.
Setup is measured separately once.

We also report the \emph{exponent ratio}
\begin{equation}
 R_E=100\,\frac{B_{\rm symbols}+B_{\rm raw}+B_{\rm metadata}}
                         {B_{\rm original\ exponents}},
 \label{eq:paper-ratio}
\end{equation}
where smaller is better. We include raw fallback and decompression metadata and exclude verbatim sign and mantissa bits. We compute the speedup as the reference time divided by \method{} time.

\begin{figure*}[!t]
\centering
\begingroup
\definecolor{trainingZipNN}{HTML}{4C78A8}
\definecolor{trainingNeuralZip}{HTML}{E45756}
\begin{tikzpicture}
\begin{groupplot}[
  group style={group size=4 by 2, horizontal sep=0.5cm, vertical sep=0.45cm},
  width=0.205\linewidth, height=0.16\linewidth, scale only axis,
  axis lines*=left, tick align=outside, tick pos=left,
  ymajorgrids, grid style={draw=gray!30, line width=0.3pt},
  scaled ticks=false,
  ticklabel style={font=\normalfont\scriptsize, /pgf/number format/fixed},
  label style={font=\normalfont\scriptsize}, title style={font=\normalfont\scriptsize},
  every axis plot/.append style={line width=0.8pt, mark size=1pt},
  unbounded coords=jump,
  legend style={draw=none, font=\normalfont\scriptsize}, legend columns=2
]
\nextgroupplot[xmin=0, xmax=2000, xtick={0,1000,2000},
  ymin=23, ymax=37, ytick={25,30,35}, title={OLMo-2 1B}, xticklabels=\empty, ylabel={Exponent ratio (\%)}, legend to name=checkpoint-training-legend]
\addplot[trainingZipNN, mark=*]
  table[x=step_k, y=zipnn_ratio, col sep=comma] {data/plots/checkpoint-bf16-olmo2-1b.csv};
\addlegendentry{ZipNN}
\addplot[trainingNeuralZip, mark=square*]
  table[x=step_k, y=neuralzip_ratio, col sep=comma] {data/plots/checkpoint-bf16-olmo2-1b.csv};
\addlegendentry{NeuralZip}
\nextgroupplot[xmin=0, xmax=650, xtick={0,300,600},
  ymin=23, ymax=37, ytick={25,30,35}, title={Tri 1.9B}, xticklabels=\empty]
\addplot[trainingZipNN, mark=*]
  table[x=step_k, y=zipnn_ratio, col sep=comma] {data/plots/checkpoint-bf16-tri-1.9b.csv};
\addplot[trainingNeuralZip, mark=square*]
  table[x=step_k, y=neuralzip_ratio, col sep=comma] {data/plots/checkpoint-bf16-tri-1.9b.csv};
\nextgroupplot[xmin=0, xmax=150, xtick={0,75,150},
  ymin=23, ymax=37, ytick={25,30,35}, title={Pythia 1B}, xticklabels=\empty]
\addplot[trainingZipNN, mark=*]
  table[x=step_k, y=zipnn_ratio, col sep=comma] {data/plots/checkpoint-bf16-pythia-1b.csv};
\addplot[trainingNeuralZip, mark=square*]
  table[x=step_k, y=neuralzip_ratio, col sep=comma] {data/plots/checkpoint-bf16-pythia-1b.csv};
\nextgroupplot[xmin=0, xmax=850, xtick={0,400,800},
  ymin=23, ymax=37, ytick={25,30,35}, title={Gaperon 1.5B}, xticklabels=\empty]
\addplot[trainingZipNN, mark=*]
  table[x=step_k, y=zipnn_ratio, col sep=comma] {data/plots/checkpoint-bf16-gaperon-1.5b.csv};
\addplot[trainingNeuralZip, mark=square*]
  table[x=step_k, y=neuralzip_ratio, col sep=comma] {data/plots/checkpoint-bf16-gaperon-1.5b.csv};
\nextgroupplot[xmin=0, xmax=2000, xtick={0,1000,2000},
  ymin=0, ymax=2.5, ytick={0,0.5,1,1.5,2,2.5},
  ylabel={Compression time (s)}]
\addplot[trainingZipNN, mark=*]
  table[x=step_k, y=zipnn_seconds, col sep=comma] {data/plots/checkpoint-bf16-olmo2-1b.csv};
\addplot[trainingNeuralZip, mark=square*]
  table[x=step_k, y=neuralzip_seconds, col sep=comma] {data/plots/checkpoint-bf16-olmo2-1b.csv};
\nextgroupplot[xmin=0, xmax=650, xtick={0,300,600},
  ymin=0, ymax=2.5, ytick={0,0.5,1,1.5,2,2.5}, yticklabels=\empty]
\addplot[trainingZipNN, mark=*]
  table[x=step_k, y=zipnn_seconds, col sep=comma] {data/plots/checkpoint-bf16-tri-1.9b.csv};
\addplot[trainingNeuralZip, mark=square*]
  table[x=step_k, y=neuralzip_seconds, col sep=comma] {data/plots/checkpoint-bf16-tri-1.9b.csv};
\nextgroupplot[xmin=0, xmax=150, xtick={0,75,150},
  ymin=0, ymax=2.5, ytick={0,0.5,1,1.5,2,2.5}, yticklabels=\empty]
\addplot[trainingZipNN, mark=*]
  table[x=step_k, y=zipnn_seconds, col sep=comma] {data/plots/checkpoint-bf16-pythia-1b.csv};
\addplot[trainingNeuralZip, mark=square*]
  table[x=step_k, y=neuralzip_seconds, col sep=comma] {data/plots/checkpoint-bf16-pythia-1b.csv};
\nextgroupplot[xmin=0, xmax=850, xtick={0,400,800},
  ymin=0, ymax=2.5, ytick={0,0.5,1,1.5,2,2.5}, yticklabels=\empty]
\addplot[trainingZipNN, mark=*]
  table[x=step_k, y=zipnn_seconds, col sep=comma] {data/plots/checkpoint-bf16-gaperon-1.5b.csv};
\addplot[trainingNeuralZip, mark=square*]
  table[x=step_k, y=neuralzip_seconds, col sep=comma] {data/plots/checkpoint-bf16-gaperon-1.5b.csv};
\end{groupplot}
\path (group c1r2.south west) -- node[midway, below=0.55cm]
  {Step ($\times 10^3$)} (group c4r2.south east);
\node[anchor=south, yshift=0.35cm] at (current bounding box.north)
  {\pgfplotslegendfromname{checkpoint-training-legend}};
\end{tikzpicture}
\endgroup
\caption{\textbf{Compression across training checkpoints.}
We keep each family's earliest setup frozen.}
\label{fig:paper-checkpoints}
\end{figure*}
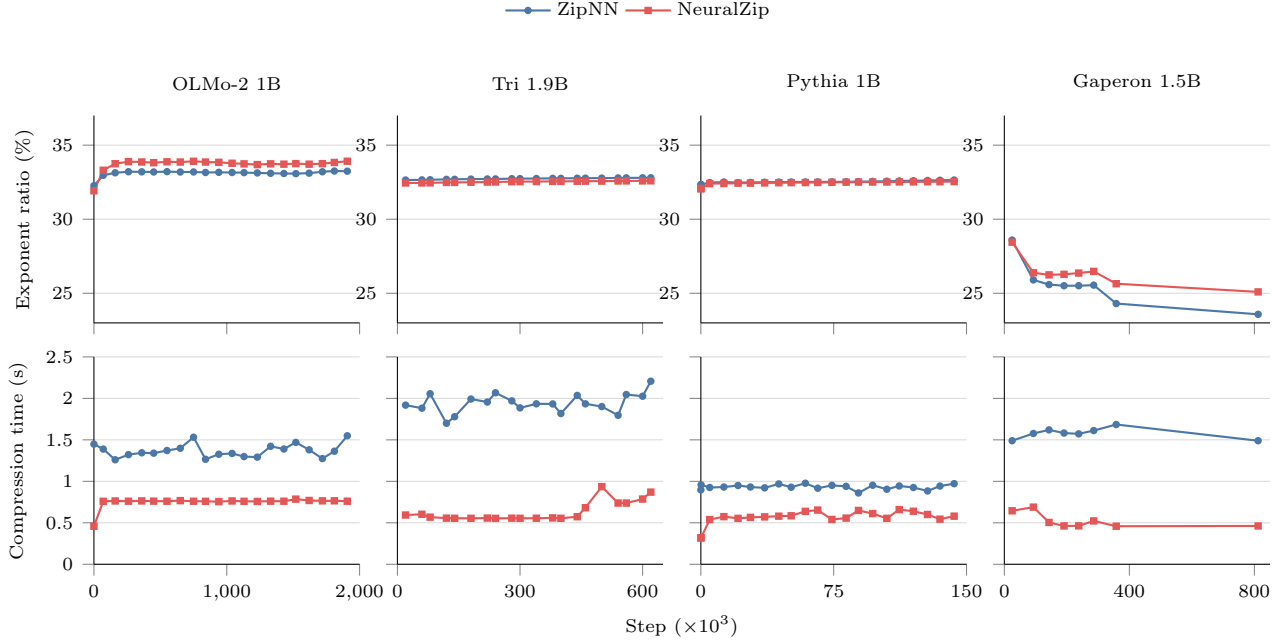

\subsection{Model compression}
The results in Tables~\ref{tab:paper-transfer} and \ref{tab:model-compression} are consistent. Compression is 1.81--21.33$\times$ faster than the fastest baseline when the setup is already available, while reductions in the exponent storage ratio range from -0.22 to 1.49 percentage points. We see higher speedups in convolutional models, at 12.24--21.33$\times$, followed by MoE models. These results indicate that the benefits of reusing codebooks vary across architectures, there is no universal pattern that allows for obtaining the same speedups every time. Instead, the results will depend on the extent of exponent heterogeneity and local exponent redundancy in the networks.

\subsection{Reusing setups}
\label{sec:paper-transfer}
Compression speedups require a precomputed setup, whose cost increases as models scale. In a model-storage setting, setups can be built once and shared with compatible targets. Using a precomputed setup from a similar model provides compression-time gains of between 1.77 and 16.65$\times$ for a 0.01--4.90 percentage-point increase in the exponent storage ratio. When a compatible frozen setup is already available, the target avoids setup construction. The source construction cost has still been incurred, and target codebook selection is included in compression time. Table~\ref{tab:model-compression} also shows cases where setup transfer worsens compression efficiency: the exponent ratio increases by up to 4.90 percentage points. Choosing a compatible source setup therefore remains important for a model hub.

\subsection{Sharing and packing contribution}

We also evaluate the two mechanisms of \method{} in isolation. \fast{} uses clustering and precomputed codebooks without packing, while \ratio{} precomputes packed codebooks for every chunk. \fast{} usually compresses faster than \method{}, but has a higher exponent storage ratio. \ratio{} improves on \fast{}'s ratio but requires more setup and compression time. Among the evaluated variants, combining both mechanisms in \method{} gives the lowest exponent storage ratio, but does not always minimize running time. Appendix~\ref{tab:fast-ratio-trade-off} shows that \ratio{} has substantial metadata costs, which sharing codebooks reduces. The mechanisms therefore complement one another: packing reduces payload size, while sharing reduces codebook overhead.

\subsection{Compressed GPU inference}
\label{sec:paper-gpu}

Tables~\ref{tab:paper-transfer} and \ref{tab:model-compression} report GPU runs:
five-trial medians of fifty synchronized forwards after twenty warm-ups. We verify weights and outputs bit-by-bit against
BF16. The runs reduce live memory up to
27.5\% with their own setups.
Own inference ranges from 11.8\% faster to 8.3\% slower than \dfloat{}. Therefore, fast decompression allows \method{} to achieve competitive inference latency while maintaining fast compression.


\subsection{Training checkpoints storage}

Setup transfer offers another possible use case: storing training checkpoints. Successive checkpoints may retain sufficiently similar exponent statistics for a setup to remain reusable. We compress public intermediate training checkpoints from OLMo-2 1B \citep{olmo2024olmo2}, Tri 1.9B \citep{trillionlabs2025tri19}, Pythia 1B \citep{biderman2023pythia}, and Gaperon 1.5B \citep{godey2025gaperon} (See Figure~\ref{fig:paper-checkpoints}).
We construct the setup of \method{} from each family's earliest available checkpoint and keep it frozen for subsequent checkpoints.
\method{} compresses 1.33--3.74$\times$ faster than \zipnn{} while maintaining similar exponent storage ratios. At later checkpoints, ratios can worsen as exponent distributions drift. Periodically refreshing the setup may help in long training runs, but these experiments do not establish a refresh schedule.


\section{Discussion and conclusion}
\label{sec:paper-discussion}
We show that reusable statistical preparation can accelerate exact compression at a similar
representation size. The heterogeneity of exponents motivates sharing codebooks,
while local exponent redundancy offers size improvements. These observations may motivate further work on reusable statistical structures for lossless compression. Deployment must account for setup construction and assess whether a precomputed setup is statistically compatible with the target weights during training and transfer.

Although our work focuses on faster compression, exponent skew has already been extensively exploited in prior work. Future work could examine source setup selection and refresh criteria while preserving bit-wise exact reconstruction.

\label{paper:last-main-page}
\subsection*{AI use statement}
In this work, we used generative AI tools for method implementation.
We have not used generative AI tools for synthetic-data generation, and qualitative analysis of human-subject data is not applicable to this work.
Additionally, we used generative AI tools for drafting and editing text, preparing figures and tables, and organizing reproducibility materials.
We have reviewed all AI-assisted work.
AI-assisted code was checked using lossless round-trip tests, bit-wise weight and output comparisons, input and setup integrity checks, and validation of recorded measurements.
We take responsibility for the final content of this work, including text, claims, or artifacts produced with the aid of generative AI.
\ifdefined\arxivversion
\subsection*{Acknowledgements}
This work was supported by the National Center for Artificial Intelligence, CENIA FB210017, Basal ANID.
\fi
\bibliography{references}
\bibliographystyle{apalike}

\ifdefined\arxivversion\else
\section*{Checklist}

\begin{enumerate}

  \item For all models and algorithms presented, check if you include:
  \begin{enumerate}
    \item A clear description of the mathematical setting, assumptions, algorithm, and/or model. [\textbf{Yes}/No/Not Applicable]
    \item An analysis of the properties and complexity (time, space, sample size) of any algorithm. [\textbf{Yes}/No/Not Applicable]
    \item (Optional) Anonymized source code, with specification of all dependencies, including external libraries. [\textbf{Yes}/No/Not Applicable]
  \end{enumerate}

  \item For any theoretical claim, check if you include:
  \begin{enumerate}
    \item Statements of the full set of assumptions of all theoretical results. [\textbf{Yes}/No/Not Applicable]
    \item Complete proofs of all theoretical results. [\textbf{Yes}/No/Not Applicable]
    \item Clear explanations of any assumptions. [\textbf{Yes}/No/Not Applicable]
  \end{enumerate}

  \item For all figures and tables that present empirical results, check if you include:
  \begin{enumerate}
    \item The code, data, and instructions needed to reproduce the main experimental results (either in the supplemental material or as a URL). [\textbf{Yes}/No/Not Applicable]
    \item All the training details (e.g., data splits, hyperparameters, how they were chosen). [\textbf{Yes}/No/Not Applicable]
    \item A clear definition of the specific measure or statistics and error bars (e.g., with respect to the random seed after running experiments multiple times). [\textbf{Yes}/No/Not Applicable]
    \item A description of the computing infrastructure used. (e.g., type of GPUs, internal cluster, or cloud provider). [\textbf{Yes}/No/Not Applicable]
  \end{enumerate}

  \item If you are using existing assets (e.g., code, data, models) or curating/releasing new assets, check if you include:
  \begin{enumerate}
    \item Citations of the creator If your work uses existing assets. [\textbf{Yes}/No/Not Applicable]
    \item The license information of the assets, if applicable. [\textbf{Yes}/No/Not Applicable]
    \item New assets either in the supplemental material or as a URL, if applicable. [\textbf{Yes}/No/Not Applicable]
    \item Information about consent from data providers/curators. [Yes/No/\textbf{Not Applicable}]
    \item Discussion of sensible content if applicable, e.g., personally identifiable information or offensive content. [Yes/No/\textbf{Not Applicable}]
  \end{enumerate}

  \item If you used crowdsourcing or conducted research with human subjects, check if you include:
  \begin{enumerate}
    \item The full text of instructions given to participants and screenshots. [Yes/No/\textbf{Not Applicable}]
    \item Descriptions of potential participant risks, with links to Institutional Review Board (IRB) approvals if applicable. [Yes/No/\textbf{Not Applicable}]
    \item The estimated hourly wage paid to participants and the total amount spent on participant compensation. [Yes/No/\textbf{Not Applicable}]
  \end{enumerate}

\end{enumerate}
\fi

\clearpage
\appendix
\onecolumn
\renewcommand{\arraystretch}{0.9}
\setlength{\abovecaptionskip}{4pt}
\setlength{\belowcaptionskip}{2pt}
\setlength{\textfloatsep}{8pt plus 2pt minus 2pt}
\setlength{\floatsep}{6pt plus 2pt minus 2pt}
\setlength{\intextsep}{6pt plus 2pt minus 2pt}
\setlength{\LTpre}{4pt}
\setlength{\LTpost}{4pt}
\thispagestyle{empty}
\aistatstitle{NeuralZip: Reusable Setup for Fast Lossless Compression\\Supplementary Materials}

\FloatBarrier

\section{Additional Experiments}
\label{app:cpu-results}



We extend the comparison to 18 models covering vision, language, audio, diffusion, and mixture-of-experts architectures. With the setup already available, \method{} compresses faster than its baselines across the reported models, generally with similar or lower exponent storage. Setup transfer avoids rebuilding codebooks on the target, but its compression ratio depends on the source model. We report each transfer source explicitly and retain cases where transfer increases storage.

\begingroup
\scriptsize
\setlength{\tabcolsep}{0.5pt}
\setlength{\LTcapwidth}{\textwidth}
\begin{longtable}{@{}llrrrrrr@{}}
\caption{\textbf{Model compression and setup reuse.}}\label{tab:model-compression}\\
\toprule
Model & Compressor & \shortstack{Setup\\(s) $\downarrow$} & \shortstack{Compression\\(s) $\downarrow$} & \shortstack{Exponent\\ratio\\(\%) $\downarrow$} & \shortstack{Total\\memory\\(GB) $\downarrow$} & \shortstack{GPU\\memory\\(GiB) $\downarrow$} & \shortstack{Inference\\(ms) $\downarrow$} \\
\midrule
\endfirsthead
\multicolumn{8}{l}{\tablename~\thetable{} (continued)}\\
\toprule
Model & Compressor & \shortstack{Setup\\(s) $\downarrow$} & \shortstack{Compression\\(s) $\downarrow$} & \shortstack{Exponent\\ratio\\(\%) $\downarrow$} & \shortstack{Total\\memory\\(GB) $\downarrow$} & \shortstack{GPU\\memory\\(GiB) $\downarrow$} & \shortstack{Inference\\(ms) $\downarrow$} \\
\midrule
\endhead
\midrule
\multicolumn{8}{@{}p{\textwidth}@{}}{}\\
\multicolumn{8}{r}{\footnotesize Continued on next page}\\
\endfoot
\bottomrule
\multicolumn{8}{@{}p{\textwidth}@{}}{}\\
\endlastfoot
ResNet-101 & BF16 & -- & -- & 100.000 & 0.089 & 0.093 & 14.73 \\*
\citep{he2016deep} & \zipnn{} & -- & 0.445 & 34.193 & \textbf{0.060} & -- & -- \\*
 & \method{} (own) & 3.449 & \textbf{0.026} & \textbf{33.776} & \textbf{0.060} & 0.071 & 16.47 \\*
 & \method{} (from ResNet-152) & -- & 0.027 & 34.129 & \textbf{0.060} & 0.073 & 16.37 \\*
 & \dfloat{} & -- & 19.216 & 37.809 & 0.062 & \textbf{0.070} & \textbf{15.78} \\
\midrule
ConvNeXt-Base & BF16 & -- & -- & 100.000 & 0.177 &  0.174&  23.43\\*
\citep{liu2022convnet} & \zipnn{} & -- & 0.554 & 33.045 & \textbf{0.118} & -- & -- \\*
 & \method{} (own) & 2.880 & 0.039 & \textbf{32.821} & \textbf{0.118} &  0.132&  16.08\\*
 & \method{} (from ConvNeXt-Large) & -- & \textbf{0.038} & 32.965 & \textbf{0.118} &  0.134&  16.22\\*
 & \dfloat{} & -- & 31.030 & 35.919 & 0.120 &  \textbf{0.130}&  \textbf{15.51}\\
\midrule
EfficientNet-B3 & BF16 & -- & -- & 100.000 & 0.025 &  0.031&  4.74\\*
\citep{tan2019efficientnet} & \zipnn{} & -- & 0.308 & \textbf{33.321} & \textbf{0.016} & -- & -- \\*
 & \method{} (own) & 2.212 & \textbf{0.014} & 33.543 & 0.017 &  0.028&  6.47\\*
 & \method{} (from EfficientNet-B7) & -- & 0.084 & 34.723 & 0.017 &  0.029&  6.60\\*
 & \dfloat{} & -- & 4.632 & 39.305 & 0.017 &  \textbf{0.027}&  \textbf{5.98}\\
\midrule
GPT-2 & BF16 & -- & -- & 100.000 & 0.249 &  0.247&  10.72\\*
\citep{radford2019language} & \zipnn{} & -- & 0.294 & 33.174 & 0.166 & -- & -- \\*
 & \method{} (own) & 2.281 & \textbf{0.044} & \textbf{32.891} & \textbf{0.165} &  0.206&  9.94\\*
 & \method{} (from GPT-2 Medium) & -- & 0.050 & 34.311 & 0.167 &  0.207&  9.97\\*
 & \dfloat{} & -- & 39.547 & 35.869 & 0.169 &  \textbf{0.205}&  \textbf{9.71}\\
\midrule
BERT Base & BF16 & -- & -- & 100.000 & 0.220 & 0.221 & 20.77 \\*
\citep{devlin2019bert} & \zipnn{} & -- & 0.349 & 32.598 & \textbf{0.146} & -- & -- \\*
 & \method{} (own) & 1.883 & \textbf{0.040} & \textbf{32.353} & \textbf{0.146} & 0.178 & 16.89 \\*
 & \method{} (from BERT Large) & -- & 0.042 & 32.937 & \textbf{0.146} & 0.179 & \textbf{15.04} \\*
 & \dfloat{} & -- & 36.998 & 35.336 & 0.149 & \textbf{0.176} & 16.97 \\
\midrule
\midrule
Qwen 2.5 7B & BF16 & -- & -- & 100.000 & 15.231 & 14.236 & 423.24 \\*
\citep{yang2024qwen25} & \zipnn{} & -- & 4.686 & 33.347 & 10.153 & -- & -- \\*
 & \method{} (own) & 267.972 & \textbf{2.568} & \textbf{32.965} & \textbf{10.127} & 11.130 & 579.56 \\*
 & \method{} (from Qwen 2.5 14B) & -- & 2.592 & 34.946 & 10.278 & 11.252 & 544.17 \\*
 & \dfloat{} & -- & 2405.024 & 36.069 & 10.362 & \textbf{11.012} & \textbf{537.46} \\
\midrule
Qwen 3.5 2B & BF16 & -- & -- & 100.000 & 4.548 & 4.130 & 168.51 \\*
\citep{qwen2026qwen35} & \zipnn{} & -- & 2.712 & 33.160 & 3.028 & -- & -- \\*
 & \method{} (own) & 38.970 & \textbf{0.698} & \textbf{32.906} & \textbf{3.023} & \textbf{3.259} & 195.74 \\*
 & \method{} (from Qwen 3.5 27B) & -- & 0.711 & 33.369 & 3.033 & 3.272 & 192.51 \\*
 & \dfloat{} & -- & 720.987 & 35.777 & 3.088 & 3.303 & \textbf{189.06} \\
\midrule
OLMo-2-7B & BF16 & -- & -- & 100.000 & 14.597 & 13.603 & 395.83 \\*
\citep{olmo2024olmo2} & \zipnn{} & -- & 5.925 & 33.619 & 9.753 & -- & -- \\*
 & \method{} (own) & 1151.162 & 2.409 & \textbf{33.021} & \textbf{9.710} & \textbf{10.455} & 531.33 \\*
 & \method{} (from OLMo-2 13B) & -- & \textbf{2.393} & 35.156 & 9.866 & 10.684 & \textbf{515.14} \\*
 & \dfloat{} & -- & 2304.563 & 37.201 & 10.014 & 10.462 & 519.05 \\
\midrule
Llama 3.2 3B & BF16 & -- & -- & 100.000 & 6.425 &  5.993&  214.55\\*
\citep{meta2024llama32} & \zipnn{} & -- & 3.139 & 32.907 & 4.270 & -- & -- \\*
 & \method{} (own) & 46.482 & \textbf{1.046} & \textbf{32.645} & \textbf{4.262} &  4.543&  265.93\\*
 & \method{} (from Llama 3.2 1B) & -- & 1.509 & 33.243 & 4.281 &  4.540&  262.73\\*
 & \dfloat{} & -- & 1005.953 & 35.499 & 4.353 &  \textbf{4.500}&  \textbf{256.91}\\
\midrule
Phi-4 14B & BF16 & -- & -- & 100.000 & 29.319 &  27.392&  659.60\\*
\citep{abdin2024phi4} & \zipnn{} & -- & 8.792 & 32.618 & 19.442 & -- & -- \\*
 & \method{} (own) & 207.349 & 4.868 & \textbf{32.348} & \textbf{19.403} &  20.013&  945.63\\*
 & \method{} (from Qwen 2.5 7B) & -- & \textbf{4.393} & 34.236 & 19.680 &  20.086&  \textbf{867.96}\\*
 & \dfloat{} & -- & 4997.372 & 35.217 & 19.822 &  \textbf{19.785}&  888.47\\
\midrule
\midrule
Gemma 3 4B & BF16 & -- & -- & 100.000 & 8.600 &  8.029&  259.32\\*
\citep{gemmateam2025gemma3} & \zipnn{} & -- & 5.966 & 32.953 & 5.717 & -- & -- \\*
 & \method{} (own) & 128.738 & \textbf{1.875} & \textbf{32.687} & \textbf{5.706} &  \textbf{6.096}&  314.46\\*
 & \method{} (from Gemma 3 12B) & -- & 2.230 & 33.742 & 5.752 &  6.140&  311.54\\*
 & \dfloat{} & -- & 1373.542 & 35.545 & 5.829 &  6.149&  \textbf{309.45}\\
\midrule
Stable Diffusion v1.5 UNet & BF16 & -- & -- & 100.000 & 1.719 &  1.623&  177.13\\*
\citep{rombach2022highresolution} & \zipnn{} & -- & 1.542 & 33.292 & 1.146 & -- & -- \\*
 & \method{} (own) & 25.691 & 0.288 & \textbf{32.969} & \textbf{1.143} &  \textbf{1.207}&  195.88\\*
\citep{podell2023sdxl} & \method{} (from SDXL Base UNet) & -- & \textbf{0.270} & 37.865 & 1.185 &  1.249&  194.34\\*
 & \dfloat{} & -- & 276.392 & 35.932 & 1.168 &  1.209&  \textbf{191.49}\\
\midrule
\midrule
Qwen3-30B-A3B & BF16 & -- & -- & 100.000 & 61.064 &  56.879&  2047.27\\*
\citep{yang2025qwen3} & \zipnn{} & -- & 73.374 & 32.685 & 40.514 & -- & -- \\*
 & \method{} (own) & 948.117 & \textbf{11.621} & \textbf{32.399} & \textbf{40.430} &  56.758&  2065.53\\*
 & \method{} (from Qwen3-30B-A3B Base) & -- & 15.836 & 32.410 & 40.434 &  56.761&  \textbf{2062.21}\\*
 & \dfloat{} & -- &  9672.456&  35.225&  41.287&  \textbf{56.734}&  2081.14\\
\midrule
DeepSeekMoE-16B & BF16 & -- & -- & 100.000 & 32.751 &  30.591&  1799.15\\*
\citep{dai2024deepseekmoe} & \zipnn{} & -- & 30.102 & 32.609 & 21.717 & -- & -- \\*
 & \method{} (own) & 182.730 & \textbf{5.569} & \textbf{32.319} & \textbf{21.671} &  22.171&  2118.82\\*
 & \method{} (from DeepSeekMoE-16B Chat) & -- & 8.624 & 32.351 & 21.676 &  22.272&  2113.50\\*
 & \dfloat{} & -- &  5103.629&  35.169&  22.135&  \textbf{21.893}&  \textbf{2067.32}\\
\midrule
DeepSeek-R1-Distill-Llama 8B & BF16 & -- & -- & 100.000 & 16.061 &  14.965&  451.71\\*
\citep{guo2025deepseekr1} & \zipnn{} & -- & 6.647 & 32.724 & 10.659 & -- & -- \\*
 & \method{} (own) & 142.823 & \textbf{2.690} & \textbf{32.452} & \textbf{10.637} &  11.572&  602.54\\*
\citep{nvidia2025nemotronnano} & \method{} (from Llama 3.1 Nemotron Nano 8B) & -- & 3.747 & 32.633 & 10.652 &  11.587&  576.20\\*
 & \dfloat{} & -- & 2558.282 & 35.297 & 10.865 &  \textbf{11.448}&  \textbf{569.47}\\
\midrule
Qwen 2.5 7B Instruct Distilled & BF16 & -- & -- & 100.000 & 15.231 & 14.236 & 408.47 \\*
\citep{rz4122025qwen25distilled} & \zipnn{} & -- & 12.079 & 33.344 & 10.153 & -- & -- \\*
\citep{yang2024qwen25} & \method{} (own) & 299.476 & 3.627 & \textbf{32.958} & \textbf{10.127} & 11.129 & 537.95 \\*
 & \method{} (from Qwen 2.5 7B) & -- & \textbf{2.591} & 33.078 & 10.136 & 11.133 & \textbf{527.77} \\*
 & \dfloat{} & -- & 2411.115 & 36.067 & 10.362 & \textbf{11.012} & 530.82 \\
\midrule
\midrule
SigLIP 2 So400m & BF16 & -- & -- & 100.000 & 2.272 & 2.148 & 162.92 \\*
\citep{tschannen2025siglip2} & \zipnn{} & -- & 2.207 & 34.983 & 1.534 & -- & -- \\*
 & \method{} (own) & 40.917 & \textbf{0.370} & \textbf{34.078} & \textbf{1.523} & \textbf{1.672} & 155.37 \\*
 & \method{} (from SigLIP 2 Giant) & -- & 0.579 & 34.456 & 1.528 & 1.679 & \textbf{153.67} \\*
 & \dfloat{} & -- & 375.631 & 37.946 & 1.567 & 1.690 & 176.16 \\
\midrule
\midrule
Whisper v3 Turbo & BF16 & -- & -- & 100.000 & 1.618 &  1.517&  327.48\\*
\citep{radford2023whisper} & \zipnn{} & -- & 1.649 & 33.861 & 1.083 & -- & -- \\*
 & \method{} (own) & 45.678 & 0.265 & \textbf{33.525} & \textbf{1.080} &  1.154&  \textbf{333.65}\\*
 & \method{} (from Whisper large-v3) & -- & \textbf{0.255} & 34.514 & 1.088 &  1.159&  333.78\\*
 & \dfloat{} & -- & 261.572 & 36.665 & 1.105 &  \textbf{1.153}&  337.13\\
\addlinespace[1pt]
\end{longtable}
\endgroup
\clearpage

\FloatBarrier
\input{appendix-ablations}
\FloatBarrier

\end{document}